\documentclass[nohyperref]{article}

\usepackage{microtype}
\usepackage{graphicx}
\usepackage{subfigure}
\usepackage{booktabs} 

\usepackage{hyperref}

\usepackage[accepted]{style}

\usepackage{amsmath}
\usepackage{amssymb}
\usepackage{mathtools}
\usepackage{amsthm}
\usepackage{algorithm}
\usepackage{algorithmic}

\usepackage[capitalize,noabbrev]{cleveref}

\theoremstyle{plain}
\newtheorem{theorem}{Theorem}[section]

\newtheorem{lemma}[theorem]{Lemma}

\theoremstyle{definition}

\theoremstyle{remark}

\usepackage{amsmath}
\DeclareMathOperator*{\argmax}{arg\,max}
\DeclareMathOperator*{\argmin}{arg\,min}

\usepackage[textsize=tiny]{todonotes}

\styletitlerunning{Generative Cooperative Networks for NLG}
\usepackage{comment}

\begin{document}

\twocolumn[
\styletitle{Generative Cooperative Networks for Natural Language Generation }

\stylesetsymbol{equal}{*}

\begin{styleauthorlist}
\styleauthor{Sylvain Lamprier}{equal,isir}
\styleauthor{Thomas Scialom}{equal,isir,recital}
\styleauthor{Antoine Chaffin}{irisa,imatag}
\styleauthor{Vincent Claveau}{irisa,cnrs}
\styleauthor{Ewa Kijak}{irisa}
\styleauthor{Jacopo Staiano}{recital}
\styleauthor{Benjamin Piwowarski}{isir,cnrs}
\end{styleauthorlist}

\styleaffiliation{isir}{ISIR - Sorbonne Université, Paris, France}
\styleaffiliation{recital}{ReciTAL, Paris, France}
\styleaffiliation{imatag}{IMATAG, Rennes, France}
\styleaffiliation{irisa}{IRISA, Rennes, France}
\styleaffiliation{cnrs}{CNRS}

\stylecorrespondingauthor{Sylvain Lamprier}{lamprier@isir.upmc.fr}

\stylekeywords{Machine Learning, style}

\vskip 0.3in
]

\printAffiliationsAndNotice{\styleEqualContribution} 

\begin{abstract}
Generative Adversarial Networks (GANs) have known a  tremendous success for many continuous generation tasks, especially in the field of image generation.  However, for discrete outputs such as  language,  optimizing GANs remains an open problem with many instabilities, as no gradient can be properly back-propagated from the discriminator output to the generator parameters.  An alternative is to learn the generator network via reinforcement learning, using the discriminator signal as a reward, but such a technique suffers from moving rewards and vanishing gradient problems. Finally, it often falls short compared to direct maximum-likelihood approaches. 
In this paper, we introduce Generative Cooperative Networks, in which the discriminator architecture is cooperatively used along with the generation policy to output samples of realistic texts for the task at hand.  We give theoretical guarantees of convergence for our approach, and study various efficient decoding schemes to empirically achieve state-of-the-art results in two main NLG tasks.

\end{abstract}

\section{Introduction}

Generative Adversarial Networks (GANs) \cite{gan} have known a tremendous success for many generation tasks. In GANs, a \textit{discriminator} network is trained to distinguish real data from fake ones, the latter being generated via a \textit{generator} network trained to fool the discriminator. Both networks are trained as a min-max two-player game, which is referred to as adversarial training. Under some strong assumptions, \cite{gan} gives theoretical guarantees of convergence of the generator towards the distribution underlying observed training data. Empirically, in continuous domains such as for image generation, these architectures have shown impressive abilities to generate realistic -- unobserved -- data, and have been extensively studied.

However, for discrete domains such as Natural Language Generation (NLG), optimizing GANs remains an open-problem, as no gradient-flow can be properly back-propagated from the discriminator output to the generator parameters. Language GANs require to be optimized via Reinforcement Learning (RL) methods, with rewards provided by discriminator networks \cite{de2019training}. Compared to classical NLG methods, such approaches have  the potential to 1)  avoid the well-known exposure bias plaguing the traditional MLE (teacher-forcing) training mode \cite{ranzato2015sequence} and 2) 
automatically discover useful metrics to optimize via RL  -- compared to manually designed ones \cite{ranzato2015sequence, paulus2017deep,answers}. 
Still, GAN-based approaches suffer from both high variance and non-stationary reward distributions, leading to many instabilities, and therefore usually fall short compared to traditional MLE approaches \cite{fallshort}.     

Theoretically sound attempts such as \cite{maligan} proposed to augment the discriminator scores with maximum-likelihood signals in order to stabilize rewards, but still suffer from high variance in practice \cite{fallshort}. Other attempts such as \cite{coldGAN} proposed to rather pick training samples close to  generative distribution modes to smooth the learning process and thus prevent abrupt changes in the reward function. However, this approach, if not employed with a carefully designed learning rate scheduler, still exhibits high instabilities when training until convergence -- making it harder to adapt the learning process for new NLG tasks or datasets.

Relying on the assumption that discrimination is easier  than generation, some recent approaches, such as \cite{das-global}, \cite{das-local} or \cite{selfgan}, have proposed to employ a cooperative decoding scheme where the discriminator network is used along with the generator to output more realistic samples. \cite{das-local} bias the standard beam search with scores provided by the discriminator network, to favor sequences that are classified as human-like texts. 
\cite{selfgan} builds upon that idea, but updates the generator at each step based on sequences generated from this augmented beam search process, in an expert-iteration \cite{expert-iteration} learning scheme. However, while such a kind of cooperative approach to produce accurate imitation learning samples 
is appealing, we argue in Section~\ref{gcn} that it may reveal particularly unstable.

To address these shortcomings, we propose to take inspiration from \cite{norouzi2016reward} which introduced \textit{Reward-augmented Maximum Likelihood} (RML), where samples to imitate are produced from a Boltzmann distribution $q(x)\propto exp(f(x)/\tau)$ with regard to a static quality metric $f(x)$ of any sample $x$ (e.g. the BLEU metric). Adapting this framework with more flexible learned -- i.e. GAN-like -- quality metrics, approaches such as \cite{selfgan} employ, at each step $t$ of the optimization process, metrics $f$ mainly depending on the current discriminator $D_t(x)$ for any sample $x$ at $t$. In this paper, we propose to rather consider $f(x)=\log(p_{t-1}(x) D_t(x))$ with $p_{t-1}$ the previous generator distribution and $D_t$ the discriminator at step $t$ (trained on samples from $p_{t-1}$), which allows us to avoid instability issues (notably due to possible catastrophic forgetting) and to present convergence guarantees under similar assumptions as those  considered in \cite{gan} for the continuous case. Then, we consider various efficient cooperative decoding approaches, which enable the practical optimization of such  training processes, mainly based on Monte-Carlo techniques and importance sampling.
 
Our contribution is threefold: 
\begin{itemize}
    \item We propose a novel formulation of GANs for the discrete setting, which exhibits unpublished theoretical convergence guarantees; 
    \item We propose practical efficient NLG training algorithms relying on these theoretical results, based on various sampling schemes and corresponding re-weightings;
    \item We present state-of-the-art results for two important NLG tasks: Abstractive Summarization and Question Generation.  
\end{itemize}

\section{Generative Cooperative Networks}
\label{gcn}

Let $p_d: {\cal Y} \rightarrow [0\,;1]$ be a target generative distribution, and assume we have access to training samples $y \sim p_d(y)$.  
The goal is to propose a training algorithm that converges towards $p_{T}(y) \approx p_d(y)$ after a given number of  $T$ iterations. 
In the following, we note $p_t: {\cal Y} \rightarrow [0;1]$ a generator distribution 
obtained at iteration $t$ of the algorithm, and $D_t: {\cal Y} \rightarrow [0\,;1]$ a discriminator 
that outputs the likelihood for an outcome $y \in {\cal Y}$ of having been generated from $p_d$ rather than from  $p_{t-1}$. 


A generic training process is given in Algorithm~\ref{algo:RMLGAN}, where $KL$ stands for the Kullback-Leibler divergence and $h$ is a composition function that outputs a sampling distribution $q_t$ from distributions given as its arguments. 
This training process unifies many different discrete GANs, (e.g., MaliGAN, SelfGAN, ColdGAN), 
as well as our present work, through the choice of 
function $h$ applied to the current discriminator $D_t$ and the previous generator $p_{t-1}$. 
Line 3 aims at finding the best possible discriminator $D_t$ given distributions $p_d$ and $p_{t-1}$, according to the classical objective to be maximized in GANs. Following the RML paradigm  introduced by \cite{norouzi2016reward}, line 4 seeks to optimize the generator distribution $p_t$ by considering the minimization of the KL divergence $KL(q_t||p_t)$, according to a fixed behavior distribution $q_t$ including feedback scores to be optimized (in our case, discriminator outputs). 
For cases where it is possible to efficiently sample from $q_t$, this is more efficient than considering  a more classical reinforcement learning objective implying the reversed $KL(p_t||q_t)$, usually subject to high variance  (e.g., via score function estimators). 

\begin{algorithm}
\caption{\textit{RML-GAN}}
\label{algo:RMLGAN}
\begin{algorithmic}[1]
\STATE \textbf{Input:} a generator $p_0 \in \cal G$, 
a discriminator family $\cal D$. 
\FOR{iteration $t$ from $1$  to $T$} 
\vspace{0.2cm}
    \STATE $D_t \leftarrow \arg\max\limits_{D \in {\cal D}} \begin{bmatrix} 
    \mathop{\mathbb{E}}\limits_{y \sim p_d(y)} \hspace{-0.3cm}[\log D(y)] \ +   \\  
    \mathop{\mathbb{E}}\limits_{y \sim p_{t-1}(y)}\hspace{-0.3cm}[\log (1-D(y))]
    \end{bmatrix}$  \\[0.4cm]
    \STATE $p_t \leftarrow \arg\min\limits_{p \in {\cal G}} KL(q_t=h(p_{t-1}, D_t) || p)$ 
\ENDFOR
\end{algorithmic}
\end{algorithm}

Let us consider a setting where $q_t\triangleq h(p_{t-1}, D_t)\propto exp(D_t)$, i.e. where the sampling distribution only considers outputs from the discriminator. This corresponds to a direct application of the work from \cite{norouzi2016reward} for the GAN setting. For the sake of analysis, we consider the case where, at a given step $t$, the generator distribution is optimal, i.e. $p_t = p_d$ over the whole support ${\cal Y}$. 
In the next step $t+1$, the optimal $D_{t+1}$ is equal to $0.5$ for any sample from $\cal Y$. 
In this case, optimizing $KL(q_{t+1}||p_{t+1})$ with $q_{t+1} \propto exp(D_{t+1})$ makes the generator diverge from the optimum $p_d$, forgetting all information gathered until that point.   
This shows that the direct adaptation of GAN to discrete outputs is fundamentally unstable.
%
While not exactly what is performed in approaches such as SelfGAN \cite{selfgan},\footnote{Since SelfGAN employs a pre-filter based on its generator to avoid complexity issues. } 
this extreme  setting illustrates instabilities that can occur with this family of recent state-of-the-art approaches. Discrimination cannot be all you need.  

Therefore, 
we rather propose to consider a slightly different and yet much smoother optimization scheme, where both the generator and the discriminator \textit{cooperate} to form the target distribution:  $q_t  \propto p_{t-1} D_t$. Such a choice for $q_t$ allows us to prove the following theorem, which gives theoretical convergence guarantees for our collaborative training process (proof given in Appendix~\ref{proof}).

\begin{theorem}
\label{theo1}
With $q_t \propto p_{t-1} D_t $, if the generator and discriminator architectures have enough capacity, and at each iteration of Algorithm~\ref{algo:RMLGAN} both optimization problems are allowed to reach their respective optimum  (i.e., 
    $D_t(y)=\frac{p_d(y)}{p_d(y)+p_{t-1}(y)}$ for any $y \in {\cal Y}$ (line 3) and 
    $KL(q_t \propto p_{t-1} D_t || p_t)=0$ (line 4)), 
then, starting from $p_0$ such that $p_0(y) > 0$ whenever $p_d(y)>0$,
$p_t$ converges in distribution to $p_d$ when $t \rightarrow +\infty$. 
\end{theorem}

As for classic continuous  GANs, the neural architectures used to define generator and discriminator function sets $\cal G$ and $\cal D$  in practice represent a limited family of distributions, depending of their depth and width. However, the given theorem allows us to expect reasonable behavior for sufficiently powerful architectures. 
The following theorem relaxes the constraint on optimal discriminator (proof in Appendix~\ref{proof2}).
\begin{theorem}
\label{theo2}
With $p_{t} \propto p_{t-1} D_{t}$, and if the  discriminator is sufficiently trained, i.e. we have $\log \eta = \min\left(\mathop{\mathbb{E}}\limits_{y \sim p_d(y)}[\log(D_{t}(y)], \mathop{\mathbb{E}}\limits_{y \sim p_{t-1}(y)}[\log(1-D_{t}(y))]\right) $,  
with $\eta \in ]\frac{1}{2};1[$,
then we have at each iteration of Algorithm \ref{algo:RMLGAN}: $\Delta_t \triangleq KL(p_d||p_{t}) - KL(p_d||p_{t-1}) \leq \log(\frac{1}{\eta}-1) <0$. 
\end{theorem}
In other words, it suffices that both parts of the discriminator objective exceed the random accuracy (i.e.,  $1/2$) in expectation to make $q_t \propto p_{t-1} D_{t}$ a useful target to be approximated at each step. Even with only a few gradient steps at each iteration, we can reasonably assume that the parameters space is smooth enough to guarantee the convergence of the algorithm, with almost only useful  gradient steps. We also note that the better discriminator (i.e., higher $\eta$), the more useful is a move from $p_{t-1}$ to $p_{t}$ (in terms of KL).

Getting back to Algorithm~\ref{algo:RMLGAN},
at line 4, optimization can be performed via gradient descent steps  $\nabla_{p_t} KL(q_t||p_t)$,  
which can be rewritten via Importance Sampling as: 
\begin{eqnarray}
\label{gradient}
\nabla_{{p_t}} KL(q_t||p_t)   
= - & \hspace*{-0.7cm}\mathop{\mathbb{E}}\limits_{y \sim q_t(y) \propto p_{t-1}(y) D_t(y)}\hspace*{-0.3cm}\left[ \nabla_{p_t} \log p_t(y) \right] \\ 
= - & \hspace*{-0.5cm}\mathop{\mathbb{E}}\limits_{y \sim p_{t-1}(y)}\left[\frac{q_t(y)}{p_{t-1}(y)} \nabla_{p_t} \log p_t(y) \right] \nonumber \\
 =  - & \hspace*{-0.5cm} \frac{1}{Z_t}\hspace*{-0.0cm}\mathop{\mathbb{E}}\limits_{y \sim p_{t-1}(y)}\hspace*{-0.3cm}\left[D_t(y) \nabla_{p_t} \log p_t(y) \right]  
 \label{gradient2}
\end{eqnarray}
with $Z_t=\sum_{y \in {\cal Y}} p_{t-1}(y) D_t(y) $ the partition function of $q_t$.
Note that, to the exception of the partition score $Z_t$ that acts as a scale at each step, the considered gradient is closely similar to what is optimized in classic discrete GANs via reinforcement learning (i.e., policy gradient optimization of $p_t$ and the discriminator score as reward, as described for instance in \cite{coldGAN}), 
when only one gradient update is performed at each iteration.

The effect of this scaling factor can be seen when written as an expectation, i.e. $Z_t=\mathop{\mathbb{E}}_{y \sim p_{t-1}(y)} [D_t(y)]$. From this, it is clear that $Z_t$ is maximized when the generator distribution coincides with $D_t$, i.e. when  $p_{t-1}$ allocates best probability mass for samples judged as the most realistic by the current discriminator. In the absence of such a normalization term,  classic GAN approaches need to set an arbitrary learning rate scheduling to avoid the explosion of gradient magnitude as $p_t$ gets closer to $p_d$. 
Our approach, naturally stabilized by $Z_t$, does not require such a difficult tuning to ensure convergence -- 
%
as verified empirically in Section~\ref{xp}. 

\section{Cooperating for NLG}

Many NLG tasks (e.g., translation, summarization, question generation, etc.) imply a context as input. This section first presents the extension of Algorithm~\ref{algo:RMLGAN} to this setting, and then discusses its practical implementation and the sampling strategies that enable its efficient use 
in real-world settings.  

\subsection{Learning algorithm} 
Let $\Gamma$ be a training set of $N$ samples $(x^i,y^i)$ where each $x^i \in {\cal X}$ is a (possibly empty) context (assumed to be sampled from a hidden condition distribution $p_x$) and $y^i \sim p_d(y^i|x^i)$ is the corresponding observation. Algorithm~\ref{algo:RMLGAN2} gives the practicable implementation of Algorithm~\ref{algo:RMLGAN} for large scale NLG tasks. It considers parametric distributions $p_\theta$ and $D_\phi$, implemented as deep neural networks,\footnote{Transformer T5 \cite{raffel2019exploring} in our experiments} with respective parameters $\theta$ and $\phi$. Thus, $p_{\theta}:{\cal X} \times {\cal Y} \rightarrow [0\,;1]$ is  the generative conditional distribution, where $p_{\theta}(y|x)=\prod_{j=1}^{|y|} p_{\theta}(y_j|x,y_{0:j-1})$ with $p_{\theta}(y_j|x,y_{0:j-1})$ the categorical distribution for token $j$ of sequence $y$ over the vocabulary, given the context $x$ and the sequence history $y_{0:j-1}$. 
Also, $D_\phi: {\cal X} \times {\cal Y} \rightarrow [0\,;1]$ is the conditional discriminative distribution, where   $D_\phi(x,y)$ returns the probability for sequence $y$ to having been generated from $p_d$ rather than $p_\theta$ 
given the context $x$. 

The discriminator is trained at line 5 of Algorithm~\ref{algo:RMLGAN2}, on a batch of $m$ samples of contexts, associated with corresponding sequences $y$ from the training set and generated sequences $\hat{y}$ from the current generator. Consistently with \citep{das-local}, to effectively drive the cooperative decoding process in guided sampling strategies $\hat{q}$ (see below), the discriminator is trained, using a classical left-to-right mask, on every possible starting sub-sequence $y_{1:j}$ in each sample $y$ (i.e., taken from its start token to its $j$-th token), with $j \leq l$ and $l$ standing for the max length for any decoded sequence.  
This enables discriminator  predictions for unfinished sequences (allowing to avoid  complex rollouts in MCTS, see below). 

Line 7 of Algorithm~\ref{algo:RMLGAN2} performs a gradient descent step for the generator, according to samples provided by a sampling strategy $\hat{q}$. Ideally, consistently with Eq.(\ref{gradient}), training samples should be provided by $q_{\theta,\phi}(y|x)\propto p_\theta(y|x) D_\phi(y|x)$. However, directly sampling from this distribution is intractable. Various sampling strategies can be considered, using a weighted importance sampling scheme to unbias gradient estimators in line 7. 
For the task  of unconditional generation (i.e., empty contexts $x$) and the case where $\hat{q}=p_\theta$, we can show that this is equivalent, up to a constant factor, to the gradient estimator given in Eq.~(\ref{gradient2}), with expectations estimated on the current batch, since in that case  $w^i$ reduces to  
$D_\phi(x^i,\hat{y}^i)$.    However, more efficient sampling strategies $\hat{q}$ can be employed, as discussed in the following.

\begin{algorithm}
\caption{\textit{Generative Cooperative Networks}}
\label{algo:RMLGAN2}
\begin{algorithmic}[1]
\STATE \textbf{Input:} generator $p_\theta$ with parameters $\theta$, discriminator $D_\phi$ with parameters $\phi$, training set $\Gamma$,  sampling strategy $\hat{q}$, batch size $m$, max sequence length $l$. 
\FOR{$t = 1,\ldots, T$}

\STATE Sample $\{(x^i,y^i)\}_{i=1}^{m}$ from $\Gamma$\;
\\[0.2cm]

\STATE $\forall i \in[\![1\,;m]\!]$: Sample $\hat{y}^i \sim p_\theta(\hat{y}^i|x^i)$; 

\STATE 
    $\phi \leftarrow \phi + \epsilon_\phi 
    \sum\limits_{i=1}^m \sum\limits_{j=1}^{l}  \begin{bmatrix}  \nabla_\phi \log D_\phi(x^i,y_{0:j-1}^i)] \ + \\   \nabla_\phi \log (1-D_\phi(x^i,\hat{y}_{0:j-1}^i)) 
    \end{bmatrix}
    $\\[0.2cm]

\STATE $\forall i \in[\![1\,;m]\!]$: Sample $\hat{y}^i \sim \hat{q}(\hat{y}^i|x^i)$; 

\STATE $\theta \leftarrow \theta + \epsilon_\theta \left[ \frac{1}{\sum_{i=1}^m w^i} \sum_{i=1}^m w^i \nabla_\theta \log p_\theta(\hat{y}^i|x^i)\right]$\;  \hspace*{2cm} with $w^i=\frac{p_{\theta}(\hat{y}^i|x^i)  D_\phi(x^i,\hat{y}^i)}{\hat{q}(\hat{y}^i|x^i)} $

\ENDFOR
\end{algorithmic}
\end{algorithm}

\subsection{Efficient Sampling}

To minimize the variance of gradient estimators, we need to sample sequences as close as possible to the distribution $q_{\theta,\phi}(y|x)\propto p_\theta(y|x) D_\phi(x,y)$. While directly sampling from such a non-parametric distribution is difficult, and given that rejection-sampling or MCMC methods are very likely to be particularly inefficient in the huge associated support domain, it is possible to build on recent advances in guided decoding for providing methods for sampling informative sequences \cite{das-local, selfgan}, that are both likely from the generator point of view $p_\theta$, and realistic from the discriminator one $D_\phi$. Note that an alternative would have been to exploit the maximum entropy principle \cite{ziebart2010modeling} to learn a neural sampling distribution $\hat{q}_\gamma$ as  $\argmax_{\hat{q}_\gamma} \mathop{\mathbb{E}}_{y \sim  \hat{q}_\gamma(y|x)}[\log p_\theta(y|x) + \log D_\phi(x,y)] + {\cal H}_{\hat{q}_\gamma(.|x)}$, with ${\cal H}_{q}$ the entropy of distribution $q$. This would however imply a difficult learning problem at each iteration of Algorithm~\ref{algo:RMLGAN2}, and a sampling distribution $\hat{q}_\gamma$ that lags far behind $q_{\theta,\phi}$ if only few optimization steps are performed.   

\subsubsection{Sampling Mixtures}

Before presenting our cooperative decoding strategy, we consider the use of variance reduction techniques when sampling from the generator distribution, which 
can be long-tailed, thus leading to unreliable sequence samples. In particular, Nucleus Sampling \cite{holtzman2019curious}  has been shown to produce higher quality texts than more classic sampling strategies, including beam search and low temperature-based sampling \cite{coldGAN}. 
Its principle is to sample tokens at each decoding step 
only from the nucleus $V_p^{(\sigma)}$ of the considered generative distribution $p$, containing a specified amount $\sigma$ of the probability mass. More precisely, let $V_p^{(\sigma)}$ be the minimal set of tokens from the vocabulary $\cal V$ whose total probability mass is greater than or equal to $\sigma$ (i.e., $V_p^{(\sigma)} = \argmin_{V \subseteq {\cal V}, \sum_{w \in V} p(w) \geq \sigma} |V|$). 
We denote in the following $p^{nucleus=\sigma}$ the truncation of distribution $p$ on the set of tokens $V_p^{(\sigma)}$. 

Using this technique for defining $\hat{q}$ in our Algorithm \ref{algo:RMLGAN2}  could allow to avoid usual text degeneration issues \cite{holtzman2019curious}, which would benefit to our generative learning process by providing better formed sequences to the discriminator. However, Importance Sampling (IS) demands that $\hat{q}(y)>0$  for any $y \in {\cal Y}$ such that $q(y)>0$. A direct use of nucleus sampling as $\hat{q}=p^{nucleus=\sigma}$, or even more a classic beam search, cannot guarantee this property, which might  involve ignoring many useful parts of  ${\cal Y}$ in the gradient estimation, hence implying biases. 

To cope with this, we propose to follow \cite{coldGAN},  which considers sampling  distributions $\hat{q}$ as mixtures, ensuring that both properties, i.e. IS consistency and high quality samples, are verified. Formally, we use
\begin{equation}
\label{mixture}
\hat{q}_\theta(y|x)=\epsilon p_\theta(y|x) + (1-\epsilon) p_\theta^{nucleus=\sigma}(y|x) 
\end{equation}
where $\epsilon$ stands for a small probability for sampling from the true generator distribution rather than using a nucleus decoding ($\epsilon=0.1$ and $\sigma=0.1$ in our experiments), 
thus ensuring the validity of our IS estimator. 
Please also note that, using such mixture trick, each IS  weight is upper-bounded by $D_\phi(y|x)/\epsilon$, which greatly limits gradient explosion issues usually associated with the use of IS in RL (or over-weighting of unlikely sequences in weighted IS).     



\subsubsection{Guided Sampling}

Next, we propose to consider cooperative decoding strategies to get a sampling distribution closer to $q_{\theta,\phi}$. More specifically, we propose to employ a Monte Carlo Tree Search strategy (MCTS), as recently considered for NLG in \cite{selfgan,leblond2021machine, DBLP:journals/corr/abs-2109-13582}. Using left-to-right decoding strategies, it can happen that all sequence candidates are judged as unrealistic by the discriminator, avoiding any useful learning signal for the generator.   MCTS allows to deal with this strong limitation of myopic decoding, by 
anticipating the final utility of the successive decisions. In MCTS, a tree is built throughout decoding by repeating the four following steps: selection, expansion, evaluation, and back-propagation. 

\paragraph{Step 1: Selection} corresponds to following a path in a tree of already explored decisions for future tokens, from its root located at the current state of the sequence to be decoded, to  a leaf $s$ of the tree, for which a value $V(s)$ has not been set yet. At each node $s$ of the tree, the child node $s'$ is selected following the \texttt{PUCT} algorithm  \citep{rosin2011multi, silver2017mastering}: 
 $$   s'=\argmax_{\hat{s}  \in child(s)} \left(V(\hat{s})+c_{puct} p_\theta(\hat{s} \mid s) \sqrt{\frac{N(s)}{1+N(\hat{s})}}\right)
$$
where $p_\theta(\hat{s} \mid s)$ corresponds to the conditional probability of sampling the next token to form sequence of $\hat{s}$ from the sequence corresponding to node $s$, 
according to the current generator probability. For children nodes $\hat{s}$ that have never been selected yet, their value $V(\hat{s})$ equals $0$. $c_{puct}$ is an hyper-parameter that controls the  exploitation/exploration trade-off of the selection process, with $N(s)$ standing for the number of times node $s$ has been selected in simulations.  

\paragraph{Step 2: Expansion} corresponds to the creation of child nodes for the identified leaf $s$, if $s$ is not terminal (end-of-sentence token). This is done in our case by restricting to tokens from the nucleus $V_{p_\theta}^{(\sigma)}$ of $p_\theta$,  as presented above. This allows to restrict the width of the tree to the most likely tokens, hence improving efficiency. 

\paragraph{Step 3: Evaluation} of the selected leaf $s$ is usually done in MCTS via a direct sampling (rollout) from $s$ to a terminal node. In our case, this is likely to imply a high variance. We thus replace rollouts by the evaluation of the corresponding unfinished sequence 
(i.e., $V(s) \leftarrow D_\phi(s)$).  

\paragraph{Step 4: Back-propagation} consists in updating values of parent nodes of $s$, to favor most promising nodes in the following selection steps of the process. Consistently with \cite{selfgan}, the value of each parent node $\tilde{s}$ of $s$ is updated as 
the maximal  score 
back-propagated to $\tilde{s}$: 
$V(\tilde{s}) \leftarrow \max(V(\tilde{s}), D_\phi(s))$. This 
led to better results than using the more classic average score from children.    

At the end of the $N$ rounds of these four steps ($N=50$ in our experiments) from a given root $r$, the next token $n$ is selected as the root's child that was the most visited 
(i.e., $\argmax_{s \in child(r)} N(s)$). 
Note that, for unconditional text generation, where no context $x$ is given to the decoder, we rather sample a child proportionally to its number of visits to maintain enough diversity during learning. 
This process is repeated using $n$ as the new root 
until reaching a terminal token or the maximum sequence length (512 in our experiments).

\paragraph{Cooperative Learning with MCTS} To use this MTCS process 
to guide the generator decoding toward sequences of high discriminator scores, in our learning Algorithm \ref{algo:RMLGAN2}, we re-use the same mixture trick as for Nucleus Sampling discussed above: 
\begin{equation}
\label{mixture2}
\hat{q}_\theta(y|x)=\epsilon p_\theta(y|x) + (1-\epsilon) p_\theta^{mcts}(y|x) 
\end{equation}
where $p_\theta^{mcts}(y|x)$ is a Dirac centered on the decoded sequence from the MCTS process in the conditional case (when contexts $x$ are available), and the MCTS sampling distribution (according to number of visits, as described in the MCTS decoding process) in the unconditional case.
%
%
Again, $\hat{q}_\theta(y|x)>0$ whenever $y \in {\cal Y}$ such that $q_{\theta,\phi}(y|x)>0$, and the IS weights are upper-bounded by $D_\phi(y|x)/\epsilon$.





\section{Experiments}
\label{xp}

\subsection{Experimental Setting}
\label{sec:ExperimentalDetails_datasets}
To evaluate the framework, we experiment on standard complementary unconditional and conditional NLG tasks, with the following datasets: 
     
     \textbf{Unconditional NLG} -- Following the same setup as in many related studies (e.g. \cite{coldGAN,fallshort}), we first compare our proposed approaches with NLG baselines on the task of unconditional text generation, where the aim is to reproduce a given generative unknown distribution of texts from samples, on the EMNLP2017 News dataset. 
    
    \textbf{Question Generation} -- The task consists in generating the question corresponding to a given text and answer. For this task,  we use the SQuAD dataset \citep{rajpurkar2016squad}, composed of 100K triplets of Wikipedia paragraphs, factual questions, and their answers. 

     \textbf{Abstractive Summarization} -- The aim of this standard sequence-to-sequence task is to produce an abstract given an input text. We use the CNN/Daily Mail dataset (CNNDM) \citep{nallapati2016abstractive}, composed of 300K news article/summaries pairs. Target  summaries consist of of multiple sentences, allowing us to evaluate models on longer texts than for the Question Generation task.

To compare the models, we consider  the standard BLEU \citep{papineni2002bleu} and ROUGE \citep{lin2004rouge} metrics.  They both are an overlap ratio between n-grams from the generated text and the ground truth. BLEU is precision oriented, while ROUGE is recall oriented.  

For the task of unconditional NLG, where diversity is of crucial importance, we 
follow \cite{fallshort}, who proposed to plot results 
as curves of BLEU (i.e., with samples classically compared to ground truth references, measuring accuracy) vs. self-BLEU (i.e., with generated samples compared to themselves, measuring diversity). 
This is done by sampling texts for various temperature settings (i.e. temperature of the softmax on top of the generator). 



We compare our models with the following baselines:

\textbf{MLE} --  
We naturally consider as an important baseline the T5 model trained via Teacher Forcing. It is furthermore used as a starting point for \textit{all} the models and baselines (unless specified). 

\textbf{ColdGAN} -- This model was one of the first GANs to outperform MLE for NLG tasks \citep{coldGAN}. Its main contribution was to introduce the use of a sampling strategy with  lowered softmax temperature during training, with the objective of stabilizing the training process. We use its best reported version, which considers a mixture with Nucleus Sampling.   

\textbf{SelfGAN} -- The work presented in \cite{selfgan} uses an expert-iteration algorithm in combination with various different  cooperative decoding strategies. In the following, we report results from its version using a MCTS process, which recently obtained state-of-the-art results on the three considered NLG tasks.

\textbf{GCN} -- Our Generative Cooperative Networks which we introduce in this paper. Three versions of Algorithm~\ref{algo:RMLGAN2} are considered in the experiments: GCN$^{\hat{q}=p}$, which corresponds to a classic GAN with implicit dynamic scheduler induced by partition $z_t=\sum_i w^i$, GCN$^{\hat{q}=Nucleus}$, which considers a mixture with Nucleus Sampling as defined in Eq.~(\ref{mixture}), and  GCN$^{\hat{q}=MCTS}$, which considers a mixture with a discriminator-guided MCTS, as defined by Eq.~(\ref{mixture2}). 

\textbf{GAN} -- For ablation study purposes, we also consider 
similar versions of our implementation of Algorithm~\ref{algo:RMLGAN2} but without the use of a normalization, respectively called GAN$^{\hat{q}=p}$, GAN$^{\hat{q}=Nucleus}$
and GAN$^{\hat{q}=MCTS}$. The normalization is replaced by a linear learning rate scheduler tuned on a validation set for GAN$^{\hat{q}=p}_{+scheduler}$, GAN$^{\hat{q}=Nucleus}_{+scheduler}$ 
and GAN$^{\hat{q}=MCTS}_{+scheduler}$. 

For each model, any decoding method could be applied at inference time, independently of the training scheme. In the following, unless specified otherwise, we report 
results obtained with a classic Beam Search decoding (with a beam size of 3)
for all the experiments.

In all our experiments, our models are initialized with the seq2seq T5 model \citep{raffel2019exploring}, trained via Teacher Forcing.
Unless specified otherwise, we use the T5-small  
architecture (60M parameters), as implemented in the HuggingFace library \citep{wolf2019huggingface}.
For our best setup, we also report the results using T5-large (3 billion parameters), denoted as T5-3B. Using 4 Nvidia V100 SXM2 GPUs, \textit{GCN}$^{\hat{q}=MCTS}$ training took 32 hours for summarization, and 8 hours for QG. This is comparable to the state-of-the-art SelfGAN model. \textit{GCN}$^{\hat{q}=Nucleus}$ only required 8 hours for training on summarization, against 2 hours for QG.





\subsection{Results and Discussion}

\paragraph{Unconditional Text Generation} Figure~\ref{fig:uncond} reports results for the unconditional NLG task. First, we observe the crucial importance of the scheduler for the GAN baselines: all of its versions without scheduler (and any normalization as in vanilla discrete GANs) strongly diverge since the first training epoch, obtaining significantly weaker results than MLE (which is the starting point of all curves from the left graph). However, we see that our GCNs are naturally implicitly scheduled, with results comparable to the scheduled version of GANs, thanks to its self-normalized IS. This is an important result, since tuning the rate scheduler from a validation set is tricky and resource consuming. We also note the significantly better and comparable behavior of GCN$^{\hat{q}=Nucleus}$ and GCN$^{\hat{q}=MCTS}$ compared to GCN$^{\hat{q}=p}$. This validates that the use of smarter sampling helps training, although the space of correct sequences is too large to fully benefit from the MCTS guided sampling.  The right graph from Figure~\ref{fig:uncond} plots accuracy vs diversity curves. Here again, we observe the significant impact of scheduling, which is naturally implied in our GCN approach, not only for the sample quality, but also on the coverage of the induced distribution. The graph also reports curves for previous GAN approaches, including \cite{maligan,seqgan}, as given by  \cite{de2019training} for the same setting, significantly under the MLE baseline.            





\begin{figure*}[!ht]
    \centering 
    \includegraphics[width=0.48\textwidth]{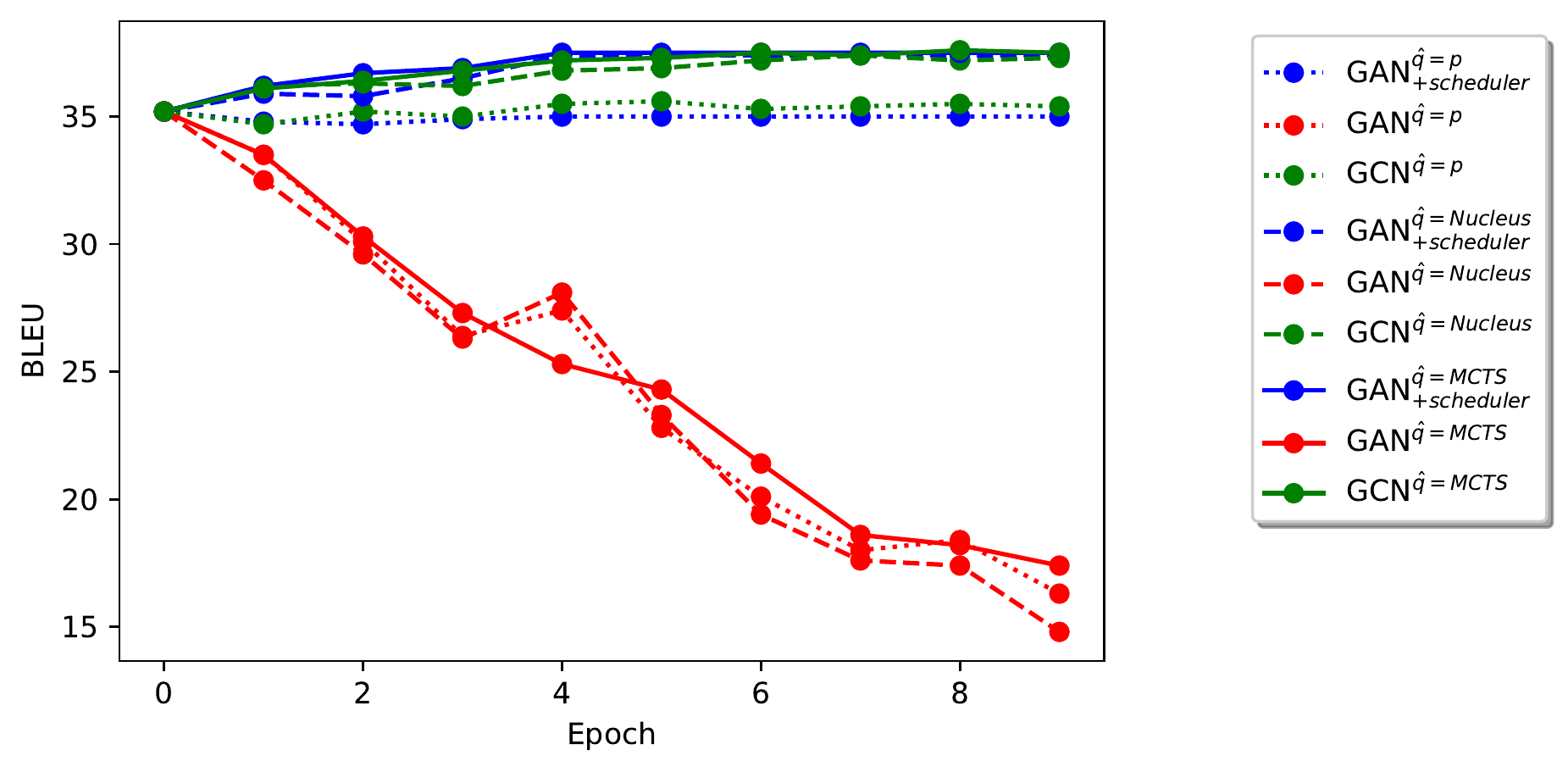}
    \includegraphics[width=0.48\textwidth]{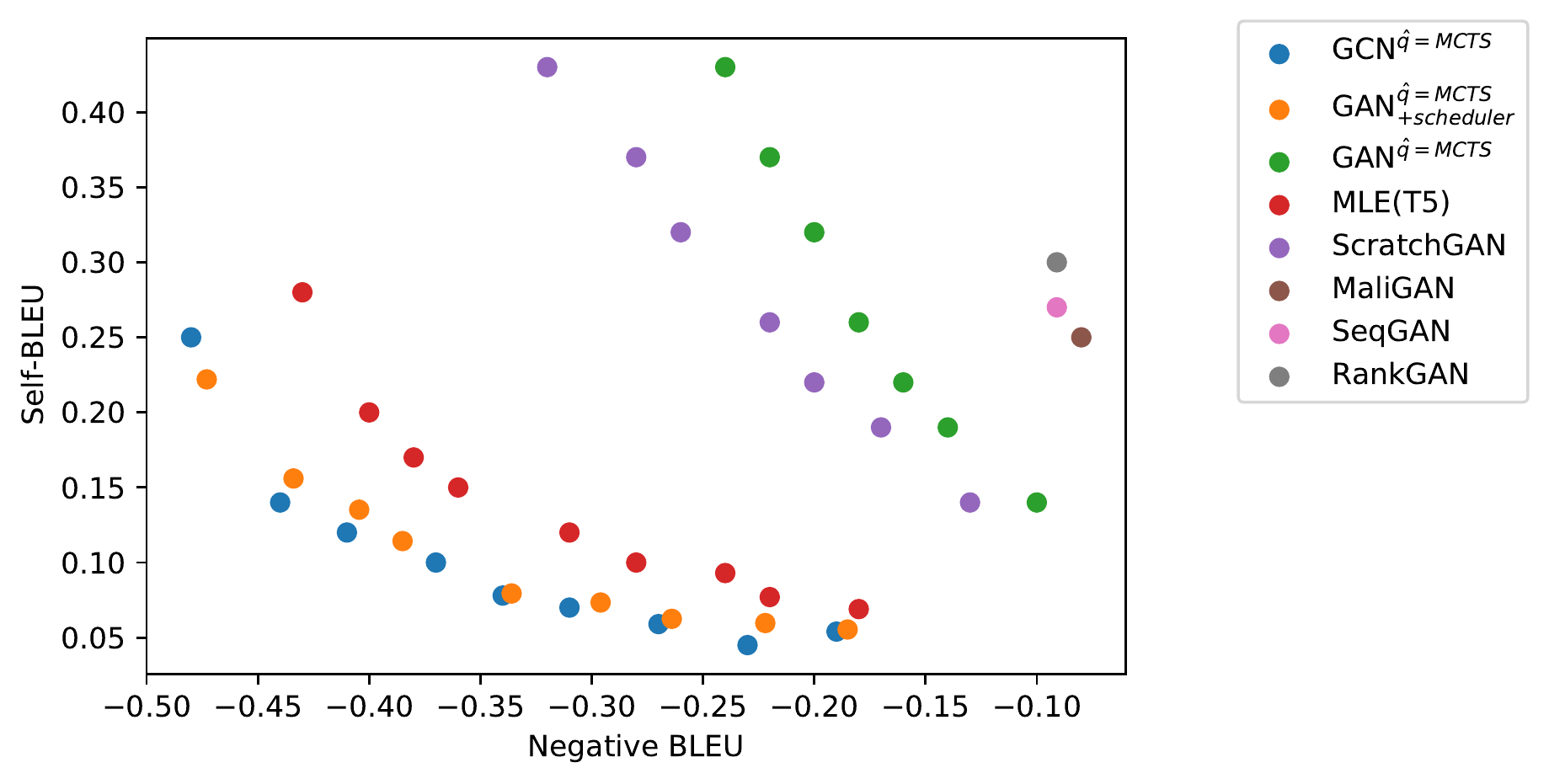}
    \vspace{-5pt}
    \caption{Results on the EMNLP 2017 dataset. Left: Evolution of BLEU results on tests sets w.r.t. training epochs (higher is better) -- red: GAN without scheduler, blue: GAN with scheduler, green: GCN. Right: Curves of negative BLEU vs self BLEU (lower is better). Scores for previous studies are taken from~\cite{de2019training}. }
    \label{fig:uncond}
\end{figure*}

\paragraph{Conditional Text Generation} More important are the results for conditional text generation, for which applications are numerous. On both considered tasks, we observe from Figure~\ref{fig:cond} the same trends as for unconditional NLG, with a  dramatic divergence of classic GAN approaches. 
We note a significant improvement of our GCN approaches compared to their GAN scheduled counterparts on both tasks, with a clear advantage for  GCN$^{\hat{q}=MCTS}$ on summarization, where the discriminator guided sampling process   obtains very stable results, significantly greater than those of other considered approaches.  This result confirms that using MCTS to sample during the learning process is key to produce long texts of better quality.  
These trends on the BLEU metrics are confirmed by numerical results from Table~\ref{tab:main_results}, where GCN$^{\hat{q}=MCTS}$ obtains the best results on both tasks over three metrics, with more than 2 ROUGE-L points gained over the very recent state-of-the-art approach SelfGAN (which also uses MCTS sampling) on QG. 
Note that these results were obtained without the complex variance reduction techniques that other RL-based GAN approaches require for obtaining results comparable to MLE, which underlines further the interest of our approach. 
For completeness, we also report results using MCTS for decoding at test time, denoted as GCN$^{\hat{q}=MCTS}_{decod=MCTS}$, which shows some further improvements, consistently with \cite{selfgan}.  
Finally, our experiment on scaling $GCN^{\hat{q}=MCTS}$ to a larger model (i.e. T5 3B instead of T5 Small) allows us to further improve the results, indicating the scaling potential for GCN, and establishing a new state-of-the-art for QG and summarization.

\begin{figure*}[!ht]
    \centering 
    \includegraphics[width=0.48\textwidth]{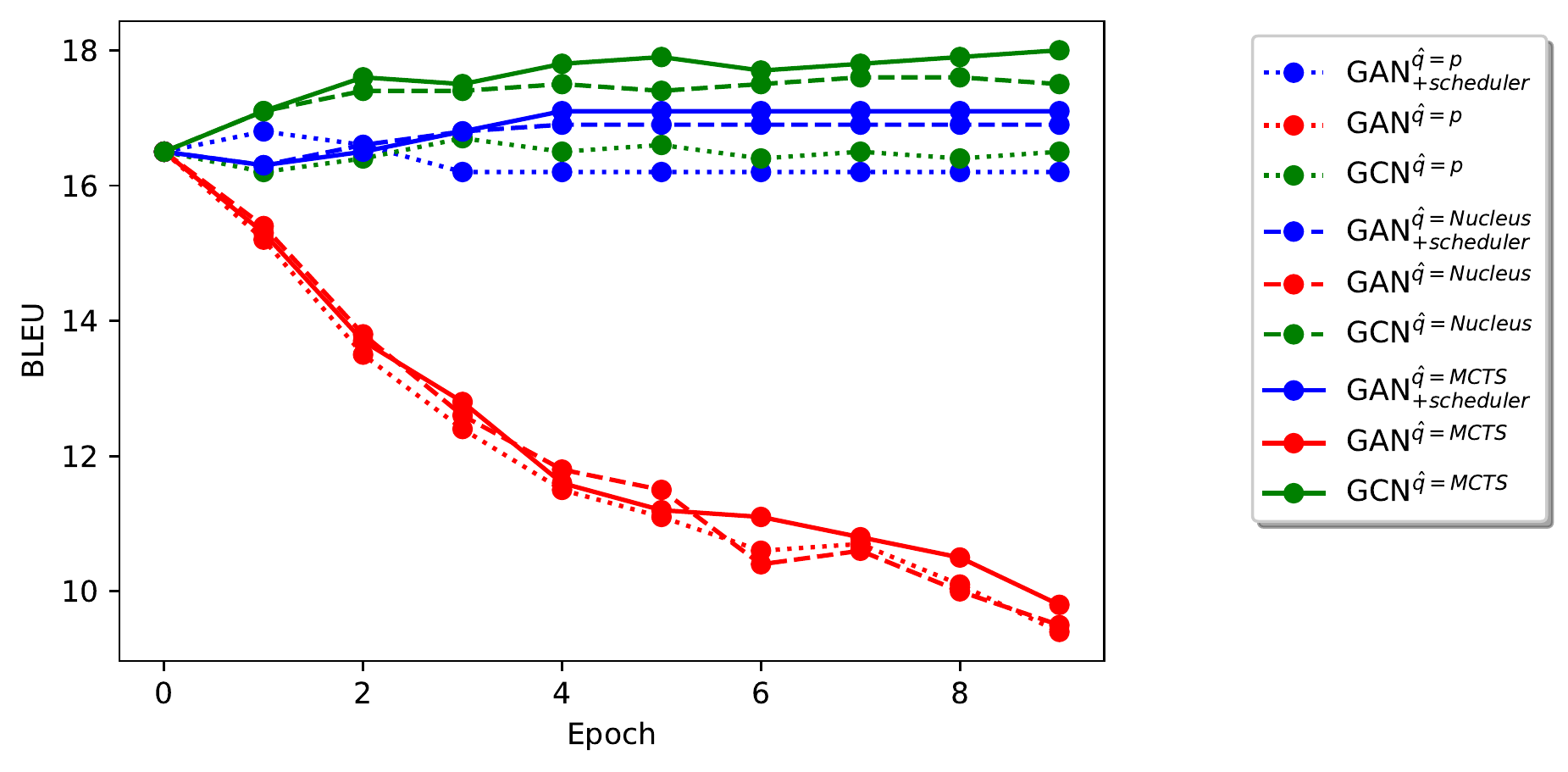}
    \includegraphics[width=0.48\textwidth]{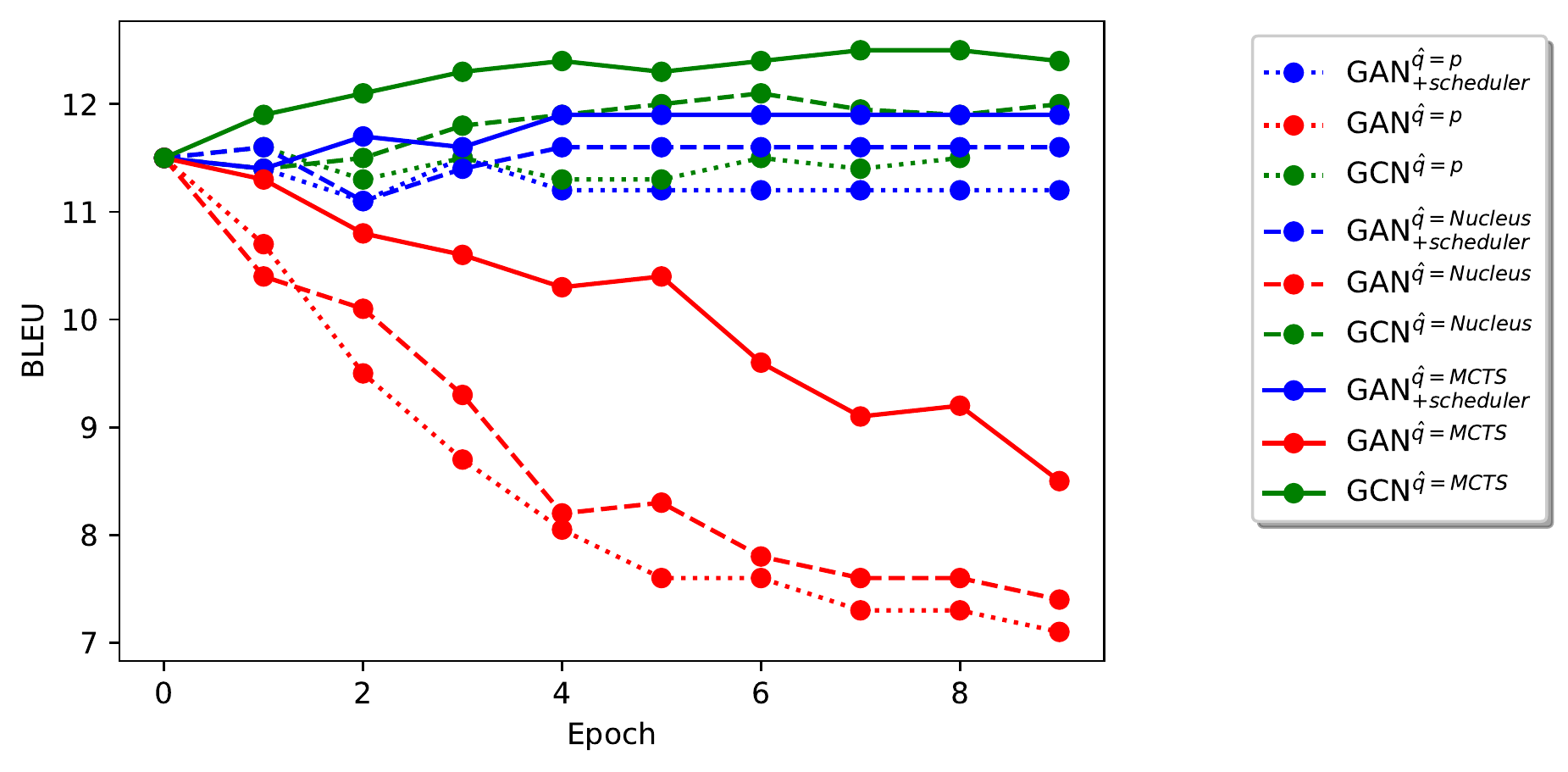}
    \vspace{-5pt}
    \caption{Evolution of performance on the test set w.r.t. training epochs (in term of BLEU, the higher the better), for conditioned NLG tasks. Left: Question Generation, Right: Summarization.} 
    \label{fig:cond}
\end{figure*}

\begin{table}[tb]
\centering
\resizebox{0.45\textwidth}{!}{
\begin{tabular}{l|ccc|ccc}
                        & \multicolumn{3}{c}{QG}                      & \multicolumn{3}{c}{Summarization}             \\
                        \toprule
                        & B        & R-1       & R-L       & B          & R-1       & R-L       \\
MLE                     & 16.5        & 43.9          & 40            & 11.5          & 36.8          & 34.9          \\
ColdGAN                 & 16.9        & 44.2          & 40.3          & 11.6          & 37.8          & 36.4          \\
SelfGAN                 & 17.2        & 44.3          & 40.6          & 12.3          & 38.6          & 36.7          \\
\midrule
GAN$^{\hat{q}=p}_{+scheduler}$           & 16.2        & 43.1          & 39.3          & 11.2          & 36.1          & 34.3          \\[0.1cm]
GAN$^{\hat{q}=p}$           & 9.4         & 25.0          & 22.8          & 7.1           & 21.0          & 19.9          \\[0.1cm]
GCN$^{\hat{q}=p}$                & 16.5        & 43.9          & 40.0          & 11.5          & 36.8          & 34.9          \\
\midrule
GAN$^{\hat{q}=Nucleus}_{+scheduler}$ & 16.9        & 45.0          & 41.0          & 11.6          & 37.7          & 35.7          \\[0.1cm]
GAN$^{\hat{q}=Nucleus}$ & 9.5         & 25.3          & 23.0          & 7.4           & 21.2          & 20.1          \\[0.1cm]
GCN$^{\hat{q}=Nucleus}$      & 17.5        & 45.3          & 42.4          & 12            & 39.0          & 37.0          \\
\midrule
GAN$^{\hat{q}=MCTS}_{+scheduler}$ & 17.1        & 45.5          & 41.5          & 11.9          & 38.1          & 36.2          \\[0.2cm]
GAN$^{\hat{q}=MCTS}$ & 9.8         & 26.1          & 23.8          & 8.5           & 21.9          & 20.7          \\[0.2cm]
GCN$^{\hat{q}=MCTS}$      & \textbf{18} & \textbf{45.9} & \textbf{42.6} & \textbf{12.4} & \textbf{39.1} & \textbf{37.1} \\
\midrule
GCN$^{\hat{q}=MCTS}_{decod=mcts}$ & \textbf{18.4} & \textbf{46.3} & \textbf{43.1} &  \textbf{12.7} & \textbf{39.4} & \textbf{37.4} \\
\midrule
GCN$^{\hat{q}=MCTS}_{\qquad T5-3B}$      & \textbf{21.8} & \textbf{49.8} & \textbf{45.9} & \textbf{19.2} & \textbf{44.2} & \textbf{43.8} \\
\bottomrule
\end{tabular}
}
\caption{\label{tab:main_results} Final results on QG and Summarization test sets, in terms of BLEU (B), ROUGE-1 (R-1) and ROUGE-L (R-L).}
\end{table}

\section{Related Work}
Under the most popular paradigm, sequence generative models~\cite{sutskever2014sequence} are usually trained with Maximum Likelihood Estimation (MLE), via Teacher Forcing~\cite{williams1989learning}. %
Though MLE has lots of attractive properties, it is however prone to overfitting for auto-regressive generative models, due to a too strong exposition to the somehow limited ground-truth  data. More importantly, MLE suffers from the mismatch between learning and simulation conditions, i.e. the well known exposure bias \cite{ranzato2015sequence, bengio2015scheduled}. Namely, at inference, the model is conditioned on sequences of previously generated tokens which may have never been observed at training time. MLE also lacks a sequence-level loss to accurately optimize sequence probabilities~\cite{welleck2019neural, negrinho2018learning}, resulting in often degenerated texts (e.g., prone to repetition)~\cite{holtzman2019curious} .

To overcome the above shortcomings of MLE, recently, many sequential GANs for discrete outputs have been proposed in the literature \cite{seqgan,guo2018long}, in which generators are typically trained to fool a learned discriminator via reinforcement learning (e.g., Policy Gradient such as the REINFORCE algorithm). While these methods allow to fill the learning-simulation gap, they usually suffer from high variance, partly due to the non-stationarity of their reward distribution. Until recently  with some advances on smoother sampling techniques and the use of control variates \cite{coldGAN},  GAN approaches usually under-performed MLE training in most real-world tasks \cite{fallshort}, with resulting sharp distributions that often sacrifice diversity for quality. Recent works based on cooperative decoding \cite{selfgan,das-local} opened the way for more efficient approaches, that rely on the discriminator not only as reward, but also for sampling, as we do in this work. However, these approaches exhibit instabilities, as discussed in Section~\ref{gcn}, which we dealt with in this paper, leveraging a more theoretically sound framework.


While not arising from the same perspective, our work on GANs for discrete outputs is strongly related to the MaliGAN approach, proposed in \cite{maligan}. Like ours, this approach 
relies on the work of  \cite{norouzi2016reward}  that unified reinforcement learning and maximum likelihood, via the consideration of a KL divergence loss between a reward-derived distribution $q$ and the learned distribution $p_t$. It extends this framework for the GAN setting by substituting  to $q$ a distribution based on a learned discriminator, 
to gain in flexibility compared to 
hand-defined metrics considered in \cite{norouzi2016reward}. 
However, rather than iteratively driving the learning process towards the data distribution $p_d$ as in this paper,  \cite{maligan} attempts to directly model it with the assumption that the discriminator is  close enough to the optimum. The approach consists in defining a reward function derived from the usual property of optimal discriminators in classic GANs, namely that $D^*(y)=p_d(y)/(p_d(y) + p_{gen}(y)$), to weight sequences (with IS) according to the unknown target distribution $p_d$. Note that this can be seen as a specific instance of our Algorithm \ref{algo:RMLGAN}, with $h(p_{t-1}, D_{t})$ defined as $p_{t-1} D_{t}/(1 - D_{t})$.   

However, the optimality of the discriminator is far from being guaranteed at each step: in \cite{coldGAN}, discriminators are shown to be strongly specialized for the current generator distribution, with possibly many sequences out of that distribution being greatly over-estimated . 
We argue that as \cite{maligan}  intrinsically relies on this optimality,
it is exposed to a high variance of its IS estimator, as acknowledged by 
the many variance reduction techniques the authors employed. Note that even doing so, they obtain comparable results to our simple GCN$^{\hat{q}=p}$ (using a pure sampling approach) for conditional NLG tasks.

In this paper, we only rely on the optimal discriminator property in our proof of convergence, similarly to continuous GANs \cite{gan}, and show that even a decent discriminator drives the convergence process in the right direction. 
%
As a further improvement over \cite{maligan}, we experimentally show, consistently with~\cite{selfgan, das-global, das-local}, that a more sophisticated discriminator-guided sampling process is highly beneficial.

%


Finally, we note that the sampling distribution $q_t$ we introduce in this work (i.e., $q_t \propto p_{t-1} D_t$) is quite similar to the energy-based generative model considered in \cite{bakhtin2021residual}, which also deals with cooperative decoding for NLG but aims to transform sequence distributions from a constant generative language model $p_\phi$, using an energy function learned by noise contrastive estimation \cite{ma2018noise}. In theory, the resulting model should match the target data distribution $p_d$, but relies on the strong  assumption that the base language model is accurate enough in the domain of $p_d$, residual learning always carrying strong liability to its base model. 
Moreover, the negative sampling considered is performed independently from the learned 
distribution, which might be particularly inefficient for long tailed distributions, with a strong divergence from the target $p_d$. 
Our work suggests that using cooperative sampling could be valuable in such a model.


\section{Conclusion}

The work presented in this paper sheds new light on discrete GAN approaches, and in particular on theoretically-sound approaches such as MaliGAN \cite{maligan}. We give a new perspective for this approach, and introduce a slightly modified algorithm, with strong theoretical guarantees, which can be combined with cooperative sampling strategies to obtain state-of-the-art results on various NLG tasks, and focused on GAN-like approaches, based on a learned discriminator to drive the generator. 
Now, it would be interesting to study how our cooperative mechanisms could apply in the context of approaches based on density ratio estimators, such as promising ones proposed in \cite{lu2019cot,song2020improving}. Hybrid approaches, based on ratio estimators between current densities and expected ones, that can be derived from our theoretical results in optimal GAN conditions, also constitute a promising research perspective, for measuring model drift 
and discovering new regularization objectives. We believe our work paves the way for new formulations of GANs for discrete settings. 
Notably, our assumption in Theorem~\ref{theo2} suggests possibly effective modifications for the discriminator loss, to gain in learning stability.



\bibliography{main}
\bibliographystyle{style}

\newpage
\appendix
\onecolumn
\section{Appendix}

\subsection{Proof for Theorem~\ref{theo1}}
\label{proof}

Let $p_d$ the data distribution that we seek at approximating, and $p_0$ be the initial generator, of same support $\cal Y$ as $p_d$ and which is not null everywhere $p_d$ is not null.  

As considered in Algorithm \ref{algo:RMLGAN}, let consider each step $t$ the learning the following discriminator optimization: 
$$D_t \leftarrow \arg\max\limits_{D \in {\cal D}} 
    \mathop{\mathbb{E}}\limits_{y \sim p_d(y)} \hspace{-0.3cm}[\log D(y)] \ +   
    \mathop{\mathbb{E}}\limits_{y \sim p_{t-1}(y)}\hspace{-0.3cm}[\log (1-D(y))]
    $$

Thus, following the proof in \cite{gan}, if $\cal D$ has enough capacity, $D_t(y)=\dfrac{p_d(y)}{p_d(y)+p_{t-1}(y)}$ for every $y \in {\cal Y}$. 

Also, at each step $t$ of Algorithm \ref{algo:RMLGAN}, we set:  
$$p_t \leftarrow \arg\min\limits_{p \in {\cal G}} KL(q_t || p)$$
With $q_t(y) \triangleq \dfrac{D_t(y) p_{t-1}(y)}{z_t}$ for each $t>0$ and all $y \in {\cal Y}$, where $z_t$ is the partition function of distribution $q_t$. 

Thus, if $\cal G$ has enough capacity and $p_t$ is sufficiently trained, we have for every $y \in {\cal Y}$ and every $t \geq 1$:
\begin{equation}
\label{eqproof1}
p_t(y) \propto D_t(y) p_{t-1}(y) = \dfrac{p_d(y) p_{t-1}(y)}{p_d(y)+p_{t-1}(y)}=\dfrac{p_d(y) }{(p_d(y)/p_{t-1}(y)) +1} \triangleq \tilde{p}_t(y) 
\end{equation}

With $z_t\triangleq \sum\limits_{y \in {\cal Y}} \tilde{p}_t(y) $, we have $p_t(y)=\dfrac{\tilde{p}_t(y)}{z_t}$.  

In the following, we consider, for all $y \in {\cal Y}$, the sequence $\hat{z}_t(y)$ defined as:
$$\hat{z}_t(y) =  \begin{cases} 
p_d(y)/p_0(y), &\text{ if } t=0; \\ 
z_t (\hat{z}_{t-1}(y) +  1), &  \forall t \geq 1.
\end{cases}$$

\begin{lemma}
\label{pt}
At every step $t$ of Algorithm \ref{algo:RMLGAN}, we have for all $y \in {\cal Y}$: 
$$
\tilde{p}_{t+1}(y) = \dfrac{p_d(y)}{\hat{z}_{t}(y)+1}$$
\end{lemma}
\begin{proof}
Let consider a proof by induction. 

First consider the base case where $t=0$. From eq.(\ref{eqproof1}), we have $\tilde{p}_1(y) = \dfrac{p_d(y) }{(p_d(y)/p_{0}(y)) +1}$ and thus, $\tilde{p}_1(y) = \dfrac{p_d(y)}{ \hat{z}_{0}(y)+1}$.

Let now assume that 
$\tilde{p}_{t}(y) =\dfrac{p_d(y)}{ \hat{z}_{t-1}(y)+1}$ is true at any step $t>0$. We need to show that this relation still holds for $t+1$ to prove the lemma.

Under this assumption, starting from Eq.(5), we have: 
\begin{eqnarray*}
\tilde{p}_{t+1}&=&
\dfrac{p_d(y) p_{t}(y) }{p_d(y) + p_{t}(y)} = \dfrac{p_d(y) \tilde{p}_{t}(y) }{p_d(y) z_t + \tilde{p}_{t}(y)} = \dfrac{p_d(y) \tilde{p}_{t}(y) }{p_d(y) z_t + p_d(y) / (\hat{z}_{t-1}(y)+1)} \\ &=&
\dfrac{\tilde{p}_{t}(y) (\hat{z}_{t-1}(y)+1) }{z_t (\hat{z}_{t-1}(y)+1) + 1  } =
\dfrac{p_{d}(y)}{z_t (\hat{z}_{t-1}(y)+1) + 1  } = \dfrac{p_{d}(y)}{\hat{z}_{t}(y) + 1  } 
\end{eqnarray*}

\end{proof}

\begin{lemma}
\label{inf1}
For every step $t>1$ of Algorithm \ref{algo:RMLGAN}, $z_t<1$.
\end{lemma}
\begin{proof}
For every step $t>0$, using lemma \ref{pt} on the second and fourth equality (below), we have: 
\begin{eqnarray*}
z_{t+1} & = & \sum_{y \in {\cal Y}} \tilde{p}_{t+1}(y) = \sum_{y \in {\cal Y}} \dfrac{p_d(y)}{\hat{z}_{t}(y) +1} = \sum_{y \in {\cal Y}} \dfrac{p_d(y)}{\hat{z}_{t-1}(y) +1} \dfrac{\hat{z}_{t-1}(y) +1}{\hat{z}_{t}(y) +1} =
\sum_{y \in {\cal Y}} \tilde{p}_{t}(y) \dfrac{\hat{z}_{t-1}(y) +1}{\hat{z}_{t}(y) +1} \\
&= &
\sum_{y \in {\cal Y}} p_{t}(y) \dfrac{z_t(\hat{z}_{t-1}(y) +1)}{\hat{z}_{t}(y) +1} =
\sum_{y \in {\cal Y}} p_{t}(y) \dfrac{\hat{z}_t(y)}{\hat{z}_{t}(y) +1} = 
\mathbb{E}_{y \sim p_{t}(y)} \left[\dfrac{\hat{z}_t(y)}{\hat{z}_{t}(y) +1}\right] 
\end{eqnarray*}
Thus, since $\hat{z}_t(y)\geq 0$ for all $y \in {\cal Y}$ and all $t \geq 0$,  $z_{t+1} <1$ for all $t>0$. 

\end{proof}

Then, to prove theorem 1 (convergence of $p_t$ to $p_d$ in law), 
let us rewrite $\hat{z}_t$ (using its definition for $t>0$) as:
$$\hat{z}_t(y)=z_t (\hat{z}_{t-1}(y) +  1)=\prod_{s=1}^t z_s (\dfrac{p_d(y)}{p_0(y)}) + \sum_{s=1}^t \prod_{s'=s}^t z_s$$

For any pair $(y,y') \in {\cal Y}^2$, we thus have: 
$$\hat{z}_t(y) - \hat{z}_t(y')
=  (\dfrac{p_d(y)}{p_0(y)} - \dfrac{p_d(y')}{p_0(y')}) \prod_{s=1}^t z_s$$

Since from Lemma \ref{inf1} we know that $z_t<1$ for any $t>1$, we have: $\lim_{t \rightarrow +\infty} \prod_{s=1}^t z_s = 0$ and thus,  $\hat{z}_t(y) - \hat{z}_t(y')$ converges to $0$ for any pair $(y,y') \in {\cal Y}^2$, ensuring that $\hat{z}_t(y)$ converges to a constant $K$, which shows that
$$\tilde p_t(y) \mathop{\rightarrow}\limits_{+\infty} \frac {p_d(y)}{1+K}$$
which in turn implies our final conclusion, i.e. that $p_t$ converges in distribution to 
$p_d$.

\subsection{Proof for Theorem~\ref{theo2}}
\label{proof2}

Let us consider the case of $p_{t} \propto p_{t-1} D_{t}$, and a  discriminator sufficiently trained such that, i.e. such that for  
\begin{equation}
\label{hypotheo2}
\log \eta  = \min\left(\mathop{\mathbb{E}}\limits_{y \sim p_d(y)}[\log(D_{t}(y)], \mathop{\mathbb{E}}\limits_{y \sim p_{t-1}(y)}[\log(1-D_{t}(y))]\right)  
\end{equation}
we have $\eta \in ]\frac{1}{2};1[$

The difference of KL divergences of the target distribution $p_d$ from the generator distribution taken at two successive steps is given as: 
\begin{eqnarray*}
\Delta_t &\triangleq& KL(p_d||p_{t}) - KL(p_d||p_{t-1}) \\
&=& \mathop{\mathbb{E}}\limits_{y \sim p_d(y)}[\log(p_{t-1}(y)) - \log(p_{t}(y))] \\
&=& \mathop{\mathbb{E}}\limits_{y \sim p_d(y)}[\log(p_{t-1}(y)) - \log(p_{t-1}(y) D_t(y))] + \log(\sum_{y' \in {\cal Y}} p_{t-1}(y) D_t(y))  \\
&=& \mathop{\mathbb{E}}\limits_{y \sim p_d(y)}[- \log( D_t(y))] + \log(\sum_{y \in {\cal Y}} p_{t-1}(y) D_t(y))\\
&=&  \log(\mathop{\mathbb{E}}\limits_{y \sim  p_{t-1}(y)}[ D_t(y)]) -  \mathop{\mathbb{E}}\limits_{y \sim p_d(y)}[ \log( D_t(y))]\\
\end{eqnarray*}

From the assumption given in Eq.(\ref{hypotheo2}),  we have: \begin{eqnarray*}
\log \eta &\leq& \mathop{\mathbb{E}}\limits_{y \sim p_{t-1}(y)}[\log(1-D_{t}(y))] \\
&\leq& \log(\mathop{\mathbb{E}}\limits_{y \sim p_{t-1}(y)}[1-D_{t}(y)]) \\
\end{eqnarray*}
where the second inequality is obtained with the Jensen inequality on expectations of concave functions.  

This equivalent to: 
\begin{eqnarray*}
\log(1-\mathop{\mathbb{E}}\limits_{y \sim p_{t-1}(y)}[1-D_{t}(y)]) \leq \log(1-\eta) 
\end{eqnarray*}

And thus: $$\log(\mathop{\mathbb{E}}\limits_{y \sim  p_{t-1}(y)}[ D_t(y)]) \leq \log(1-\eta)$$

From assumption of Eq.\ref{hypotheo2}, we also know that $\mathop{\mathbb{E}}\limits_{y \sim p_d(y)}[ \log( D_t(y))] \geq \log(\eta) $. 

Thus, we have: 
$$\Delta_t \leq \log(1-\eta) - \log(\eta) = \log(\frac{1}{\eta} - 1) < 0$$
which concludes the proof.

\end{document}